\documentclass[english,a4paper,12pt]{article}

\usepackage{fullpage}

\usepackage{amsxtra, amsfonts, amssymb, amstext, amsmath, mathtools}
\usepackage{amsthm}
\usepackage{booktabs}
\usepackage{nicefrac}
\usepackage{xspace}
\usepackage{url}
\usepackage{graphics,color}
\usepackage[algo2e,ruled,vlined,linesnumbered]{algorithm2e}
\usepackage{wrapfig}
\usepackage{lmodern}

\newtheorem{theorem}{Theorem}

\newtheorem{Lemma}[theorem]{Lemma}

\DeclareMathOperator{\cross}{cross}

\DeclareMathOperator{\Bin}{Bin}

\newcommand{\om}{\textsc{OneMax}\xspace}
\newcommand{\onemax}{\om}

\newcommand{\leadingones}{\textsc{LeadingOnes}\xspace}
\newcommand{\oea}{\mbox{${(1 + 1)}$~EA}\xspace}

\newcommand{\opllga}{\mbox{${(1+(\lambda,\lambda))}$~GA}\xspace}
\newcommand{\ollga}{\opllga}

\newcommand{\jump}{\textsc{Jump}\xspace}

\newcommand{\R}{\ensuremath{\mathbb{R}}}

\newcommand{\N}{\ensuremath{\mathbb{N}}} 
\newcommand{\Z}{\ensuremath{\mathbb{Z}}}

\newcommand{\gsemo}{GSEMO\xspace}
\newcommand{\opllgsemo}{\mbox{$(1+(\lambda,\lambda))$~GSEMO}\xspace}
\newcommand{\ollgagsemo}{\opllgsemo}

\newcommand{\oneminmax}{\textsc{OneMinMax}\xspace}

\let\originalleft\left
\let\originalright\right
\renewcommand{\left}{\mathopen{}\mathclose\bgroup\originalleft}
\renewcommand{\right}{\aftergroup\egroup\originalright}

\begin{document}

\title{ The $(1+(\lambda,\lambda))$ Global SEMO Algorithm\thanks{Author-generated version of~\cite{DoerrHP22}}} 

\author{Benjamin Doerr\footnote{Laboratoire d'Informatique (LIX), CNRS, \'Ecole Polytechnique, Institut Polytechnique de Paris, Palaiseau, France}
\and
Omar El Hadri\footnote{\'Ecole Polytechnique,Institut Polytechnique de Paris, Palaiseau, France}
\and
Adrien Pinard\footnote{\'Ecole Polytechnique,Institut Polytechnique de Paris, Palaiseau, France}
  }

\maketitle

{\sloppy
\begin{abstract}
 The $(1+(\lambda,\lambda))$ genetic algorithm is a recently proposed single-objective evolutionary algorithm with several interesting properties. We show that its main working principle, mutation with a high rate and crossover as repair mechanism, can be transported also to multi-objective evolutionary computation. We define the $(1+(\lambda,\lambda))$ global SEMO algorithm, a variant of the classic global SEMO algorithm, and prove that it optimizes the \oneminmax benchmark asymptotically faster than the global SEMO. Following the single-objective example, we design a one-fifth rule inspired dynamic parameter setting (to the best of our knowledge for the first time in discrete multi-objective optimization) and prove that it further improves the runtime to $O(n^2)$, whereas the best runtime guarantee for the global SEMO is only $O(n^2 \log n)$. 
\end{abstract}

\section{Introduction}

The theory of evolutionary algorithms (EAs) for a long time has accompanied our attempts to understand the working principles of evolutionary computation~\cite{NeumannW10,AugerD11,Jansen13,DoerrN20}. In the recent years, this field has not only explained existing approaches, but also proposed new operators and algorithms. 

The theory of \emph{multi-objective} EAs, due to the higher complexity of these algorithms, is still lagging behind its single-objective counterpart. There are several runtime analyses for various multi-objective EAs which explain their working principles. Also, some new ideas specific to multi-objective evolutionary algorithms (MOEAs) have been developed recently. However, many recent developments in single-objective EA theory have not been exploited in multi-objective evolutionary computation (see Section~\ref{sec:previous} for more details). 

In this work, we try to profit in multi-objective evolutionary computation from the ideas underlying the \ollga. The \ollga, proposed first in~\cite{DoerrDE15}, tries to combine a larger radius of exploration with traditional greedy-style exploitation of already detected profitable solutions. To this end, the \ollga uses mutation with a higher-than-usual mutation rate together with a biased crossover with the parent, which can repair the unwanted effects of the aggressive mutation operations.

We defer the detailed discussion of this algorithm to Section~\ref{sec:algo} and note here only that several results indicate that this basic idea was successful. In the early works on the \onemax benchmark, the \ollga was shown to have a better runtime than the \oea for a decent range of parameters. With the optimal parameter setting, this runtime becomes $\Theta\big(n \big(\frac{\log(n) \log\log\log(n)}{\log\log(n)}\big)^{1/2}\big)$. While this is not a radical improvement over the $\Theta(n \log n)$ runtime of many classic EAs, it is still noteworthy in particular when recalling that no black-box algorithm can optimize this benchmark faster than in time $\Theta(n / \log n)$~\cite{DrosteJW06}. With a suitable fitness-dependent parameter setting or a self-adjusting parameter choice following a 1/5-rule, the runtime of the \ollga can further be lowered to $O(n)$~\cite{DoerrD18}. Similar results have been obtained for random satisfiability instances in a planted solution model~\cite{BuzdalovD17}. Together with a heavy-tailed choice of the parameters, the \ollga can also obtain a linear runtime~\cite{AntipovBD20gecco}. Stronger runtime improvements over many classic algorithms have been obtained on the \jump benchmark, see~\cite{AntipovBD21gecco} and the references therein. Since here the right choice of the parameters is less understood and since a multi-objective analogue of the \jump-benchmark has been proposed~\cite{DoerrZ21aaai} only very recently, in this first work on multi-objective versions of the \ollga we shall concentrated on the more established \oneminmax problem, which is a multi-objective analogue of the classic \onemax benchmark.

\textbf{Our results:} We develop a multi-objective version of the \ollga by interpreting the main loop of the \ollga as a complex mutation operator and then equipping the classic \emph{global simple evolutionary multi-objective optimizer (GSEMO)} with this mutation operator. For this algorithm, which we call $\opllgsemo$, we conduct a mathematical runtime analysis on the \oneminmax benchmark. We show that a reasonable range of static parameter settings give a better runtime than the $O(n^2 \log n)$ guarantee known for the classic GSEMO. With an optimal parameter choice we obtain a runtime of $O(n^2 \sqrt{\log n})$. With a suitable state-dependent parameter choice comparable to the one of~\cite{DoerrDE15}, the runtime drops further to $O(n^2)$. Since such a state-dependent parameter choice requires a deep understanding of the algorithm and the problem, we then design a self-adjusting parameter setting inspired by the classic one-fifth success rule. Some adjustments are necessary to work in this multi-objective setting, but then we obtain a simple dynamic parameter setting that gives  the same $O(n^2)$ runtime as with the complicated state-dependent parameter choice. To the best of our knowledge, this is the first time that such a dynamic parameter choice is developed for a MOEA for a discrete search space. From a broader perspective, this work shows that the main idea of the \ollga, so far only used with $(1+1)$-type algorithms, can also be used in more complex algorithmic frameworks. 

This work is organized as follows. In the following section, we describe the previous works most relevant to ours. In Section~\ref{sec:prelim}, we brief{}ly recall the basic notations of MOEAs and state the \oneminmax problem. We develop the \opllgsemo in Section~\ref{sec:algo} and conduct our mathematical runtime analysis for static parameters in Section~\ref{sec:runtime1}. In Section~\ref{sec:statedep}, we propose state-dependent parameters and prove the $O(n^2)$ runtime guarantee for these. We develop and analyze the self-adjusting parameter choice inspired by the one-fifth rule in Section~\ref{sec:runtime2}. A short experimental evaluation in Section~\ref{sec:experiments} shows that already the \ollgagsemo with static parameters easily outperforms the classic \gsemo and this already for small problem sizes. We summarize our findings and discuss future research ideas in Section~\ref{sec:conclusion}.

\section{Previous Work}\label{sec:previous}

Soon after the first runtime analyses of single-objective EAs have appeared, see, e.g.,~\cite{Rudolph97,GarnierKS99,DrosteJW02} for three early and influential works, also multi-objective EAs were studied under this theory perspective. These works, e.g.,~\cite{LaumannsTZWD02,Giel03,LaumannsTZ04}, followed the example of the theoretical works on single-objective EAs and proved estimates on the runtime of simple MOEAs such as the SEMO and GSEMO on benchmark problems that were defined as multi-objective analogues of classic benchmarks like \onemax or \leadingones.

In the recent past, the theory of MOEA has more concentrated on research topics which are specific to multi-objective optimization, e.g., parent selection schemes that speed up the exploration of the Pareto front~\cite{OsunaGNS20}, approximations of the Pareto front~\cite{BrockhoffFN08,DoerrGN16}, or specific MOEAs such as the MOEA/D or the NSGA-II~\cite{LiZZZ16,HuangZCH19,HuangZ20,ZhengLD22,ZhengD22gecco,BianQ22arxiv,DoerrQ22arxiv}. While it is natural that the theory of MOEAs has regarded these questions, at the same time this carries the risk that trends and insights from the general EA theory are not exploited in multi-objective evolutionary computation. In fact, this effect is already very visible. Topics such as precise runtime analyses (see, e.g.,~\cite{AntipovD21telo} and the references therein), fixed-budget analysis~\cite{JansenZ14}, and optimal parameter settings (see, e.g.,~\cite{Witt13} for an early such result in single-objective EA theory) have not been considered yet in multi-objective EA theory. More critically, also the recently developed new algorithmic building blocks, for example, a large number of successful ways to dynamically set parameters~\cite{DoerrD20bookchapter}, memetic algorithms with proven performance guarantees (see~\cite{NguyenS20} and the references therein), and hyperheuristic approaches  with proven guarantees (see~\cite{LissovoiOW19} and the references therein) have all not been considered for MOEAs. In fact, to the best our our knowledge, the only example of a work transporting recent algorithmic ideas developed in single-objective EA theory to the multi-objective world is~\cite{DoerrZ21aaai}, where the fast mutation of~\cite{DoerrLMN17} and the stagnation detection of~\cite{RajabiW20} are used in a MOEA.

For this reason, we shall study in this work how another algorithmic idea recently proposed in EA theory can be used in multi-objective evolutionary computation, namely the \ollga. We defer an account of the previous work on this algorithm to Section~\ref{sec:algo}, where this algorithm will be detailed. 

\section{Preliminaries}\label{sec:prelim}

In this section, we give a brief introduction to multi-objective optimization and to the notation we use. 

For $a,b \in \R$, we write $[a..b] = \{z \in \Z \mid a \le z \le b\}$ for the integers in the interval $[a,b]$. We denote the binomial distribution with parameters $n$ and $p$ by $\Bin(n,p)$ and we write $X \sim \Bin(n,p)$ to denote that $X$ is a sample from this distribution, that is, that $\Pr[X = i] = \binom{n}{i} p^i (1-p)^{n-i}$ for all $i \in [0..n]$.

For the ease of presentation, in the remainder we shall concentrate ourselves on two objectives which have to be maximized. 
A bi-objective function on the search space $\Omega$ is a pair $f = (f_1, f_2)$, where each $f_i : \Omega \to \R$. We write $f(x) = (f_1(x),f_2(x))$ for all $x \in \Omega$. We shall always assume that we have a bit-string representation, that is, $\Omega = \{0,1\}^n$ for some $n \in \N$. The challenge in multi-objective optimization is that often there is no solution $x$ that maximizes both $f_1$ and $f_2$. 

We say \emph{$x$ weakly dominates~$y$}, denoted by $x \succeq y$, if and only if $f_1(x) \ge f_1(y)$ and $f_2(x) \ge f_2(y)$. We say \emph{$x$ strictly dominates $y$}, denoted by $x \succ y$, if and only if $f_1(x) \ge f_1(y)$ and $f_2(x) \ge f_2(y)$ and at least one of the inequalities is strict. We say that a solution is Pareto-optimal if it is not strictly dominated by any other solution. The set of objective values of all Pareto optima is called the Pareto front of~$f$.
In this work, as in many previous works on the (G)SEMO family of algorithms, our aim is to compute the full Pareto front, that is, to compute a set $P$ of Pareto optima such that $f(P) = \{f(x) \mid x \in P\}$ is the Pareto front.

In this work, we will mainly be interested in the \oneminmax benchmark function, which is a bi-objective analogue of the classic \onemax benchmark. It is defined by
$$
\oneminmax : \{0,1\}^n \to \R^2;\\
$$$$
x \mapsto \left(\sum_{i=1}^n x_i,\, \sum_{i=1}^n (1-x_i)\right),
$$
that is, the first objective counts the number of ones in $x$ and the second objective counts the number of zeros. We immediately see that any $x \in \{0,1\}^n$ is Pareto optimal. Hence the Pareto front of this problem is $\{(i,n-i) \mid i \in [0..n]\}$.

\section{From the \ollga to the \opllgsemo}\label{sec:algo}

In this section, we design a MOEA building on the main ideas of the \ollga. We start with a brief description of the \ollga and the known results on this algorithm, then move on to the GSEMO, and finally discuss how to merge the two. 

\subsection{The \ollga}\label{sec:ollga}

The \ollga is a still simple genetic algorithm for the maximization of pseudo-Boolean functions, that is, functions $f : \{0,1\}^n \to \R$ defined on the set of bit-strings of length~$n$. 
The \ollga works with a parent population of size one. Its parameters are the offspring population size $\lambda \in \N$, the mutation rate $p \in [0,1]$, usually parameterized as $p = k/n$ for some number $k \in [0,n]$, and the crossover bias $c \in [0,1]$. In a first mutation stage, from the parent individual $\lambda$ offspring are created. Each offspring is distributed as if obtained from bit-wise mutation with mutation rate $p$, however, to remedy the risk that offspring are better or worse just because they have a different distance from the parent, all offspring are created in the same distance. Consequently, first a number $\ell$ is sampled according to a binomial distribution with parameters $n$ and $p$ and then $\lambda$ offspring are generated each by flipping a random set of exactly $\ell$ bits in the parent. In the intermediate selection stage, an offspring with maximal fitness is selected (breaking ties randomly).

Due to the usually high mutation rate used in the \ollga, this mutation winner will typically be much worse than the parent. Being the best among the offspring, we can still hope that besides all destruction through the aggressive mutation, it has also gained some advantages. The crossover stage now aims at preserving these advantages and repairing the unwanted destruction. To this aim, in the crossover phase $\lambda$ offspring are created from a biased crossover between mutation winner and original parent. This biased crossover takes bits from the mutation winner with probability $c$, otherwise from the parent. In the final selection stage, the best crossover offspring is taken as new parent, except if it is worse than the original parent, in which case the original parent is kept. The pseudocode of this algorithm is given as Algorithm~\ref{alg:ollga}. 

\begin{algorithm2e}
    \caption{The \ollga with offspring population size~$\lambda$, mutation rate~$p = k/n$, and crossover bias~$c$ for the maximization of a function $f : \{0,1\}^n \to \R$.}
    \label{alg:ollga}
    Sample $x \in \{0,1\}^n$ uniformly at random\;
    \For{$t = 1,2,3,...$}
    {
        Sample $\ell$ from a binomial distribution $\mathcal{B}(n,\frac{k}{n})$\;
        
        Generate $x_1,x_2,...,x_{\lambda} \in \{0,1\}^n $ each by flipping $\ell$ random bits of $x$\;
        
        Select $x^+ \in \{x_1,x_2,...,x_{\lambda}\}$ such that $x^+$ maximizes $f$\;
        
        Generate $x^+_1,x^+_2,...,x^+_{\lambda} \in \{0,1\}^n $ via $\cross_c(x,x^+)$\;
        
        Select $y \in \{x^+_1,x^+_2,...,x^+_{\lambda}\} $ such that $y$ maximizes $f$\;
        
        \If{$f(y)\geq f(x)$}{$ x \leftarrow y$;}
        
    }
\end{algorithm2e}

It is not completely understood how to optimally set the parameters for this algorithm, but the original work~\cite{DoerrDE15} proposed to take as mutation rate $p = \lambda / n$ and as crossover bias $c = 1 / \lambda$. This setting is natural in the sense that an individual created by mutation and subsequent crossover with the parent has the same distribution as if generate with bit-wise mutation with mutation rate $1/n$, which is the common way of performing mutation. We shall call this the \emph{standard parameter setting} for the \ollga. 

In a first runtime analysis of the \ollga, it was shown that the \ollga with parameters $\lambda \ge 2$ as well as $p = k/n$ and $c = 1/k$ for some $k \ge 2$ optimizes the \onemax benchmark in time $O((\frac 1k + \frac 1\lambda) n \log n + (k+\lambda) n)$. This bound is minimized for $\lambda = k$, that is, the standard parameter setting, and further for $\lambda = k = \Theta(\sqrt{\log n})$. In this case, the runtime guarantee becomes $O(n \sqrt{\log n})$, which is asymptotically faster than the $\Theta(n \log n)$ runtime of many evolutionary algorithms~\cite{Muhlenbein92,DrosteJW02,JansenJW05,Witt06,LehreW12,RoweS14,SudholtW19,AntipovDY19,AntipovD21algo}. The truly optimal parameters, determined in~\cite{DoerrD18}, minimally deviate from these, but they also belong to the standard setting and they only improve the runtime by a $\Theta\big(\big(\frac{\log\log\log n}{\log\log n}\big)^{1/2}\big)$ factor, so we skip the details. 

In~\cite{DoerrDE15}, also a fitness-dependent parameter choice was proposed. If, in the standard setting, $\lambda$ is chosen as $\sqrt{\frac{n}{n-f(x)}}$, where $f(x)$ is the fitness of the current solution, then the runtime reduces to $O(n)$.

Since both the best static and this fitness-dependent parameter setting are non-trivial to find, in~\cite{DoerrD18} a self-adjusting parameter choice was proposed. If, again in the standard setting, $\lambda$ is controlled via a variant of the $1/5$ success rule, then the actual value of $\lambda$ follows closely the fitness-dependent choice from~\cite{DoerrDE15}, resulting also in an $O(n)$ runtime. These first results (apart from the very precise optimal static parameter setting of~\cite{DoerrD18}) are the basis of our work.

For this reason, we now describe in less detail the remaining previous works on the \ollga. Results similar to the ones just described were shown in~\cite{BuzdalovD17} for the optimization of certain random SAT instances. On \onemax again, a heavy-tailed choice of $\lambda$ also gives a linear runtime~\cite{AntipovBD20gecco}. On the \leadingones benchmark, the \ollga does not have a runtime advantage over classic EAs, however, with all parameter choices in the standard setting it also obtains the $O(n^2)$ runtime of these algorithms~\cite{AntipovDK19foga}. On the multimodal jump benchmark, a perfect understanding of the \ollga has not yet been obtained, but the existing results~\cite{AntipovDK20,AntipovBD20ppsn,AntipovBD21gecco,FajardoS20} show that with parameters outside the standard regime or a heavy-tailed parameter choice, runtimes of roughly $n^{m/2}$ can be obtained on jump functions with jump size $m$, which is significantly faster than the $\Theta(n^m)$ time of, e.g., the \oea~\cite{DrosteJW02}.

\subsection{The GSEMO}

The \emph{global simple evolutionary multi-objective optimizer (GSEMO)}~\cite{Giel03}, different from the SEMO algorithm~\cite{LaumannsTZ04} only in that is uses global mutations, is an algorithm that computes the full Pareto front of a multi-objective optimization problem. It uses a population of variable size which always consists of solutions that are incomparable, that is, no solution dominates another one. The algorithm starts with a population consisting of a single random individual. In its main loop, from a parent randomly chosen from the population an offspring is created via bit-wise mutation with mutation rate $1/n$. Any individual dominated by the offspring is removed from the population, then the offspring is added to the population if it is not dominated by a member of the population. The number of iterations this algorithm takes to find the full Pareto front of a multi-objective problem is called the \emph{runtime} of the \gsemo. 

The runtime of the \gsemo on the \oneminmax benchmark, the natural bi-objective version of the \onemax problem, is $O(n^2 \log n)$. This bound was shown for the SEMO only~\cite{GielL10}, but it is easy to see that the proof extends to the GSEMO. A lower bound of $\Omega(n^2 \log n)$ also exists only for the SEMO; this proof, however, does not immediately extend to the GSEMO.

\subsection{Designing the \opllgsemo}

A closer look at the \ollga reveals that we can interpret this algorithm as a variant of the \oea that uses a complex mutation operator. This mutation operator creates $\lambda$ offspring from a given parent, selects a best of these (``mutation winner''), creates $\lambda$ offspring via a biased crossover between parent and mutation winner, and returns a best of these. 

With this view, the natural way to merge the GSEMO and the \ollga is to use the GSEMO with this complex mutation operator instead of bit-wise mutation. Since our complex mutation operator relies on a fitness function, there is one more design choice to take, namely what to use as fitness in a multi-objective setting. Again, there is a natural choice, and this is to use all objectives the multi-objective problem is composed of. Hence in the mutation phase, we once generate $\lambda$ offspring, then for each objective we select a mutation winner, and then, separately for each of them, we conduct a crossover phase. For a $d$-objective problem, this yields $d\lambda$ crossover offspring. For all of these, we check if it is profitable to add them to the population, more precisely, sequentially for each of them we remove the individuals it dominates and add it to the population if it is not dominated by a current member of the population. The pseudocode for this algorithm, which we call \opllgsemo, is given in Algorithm~\ref{alg:opllgsemo}. 

We note that in all algorithms, we did not specify a termination criterion. This is justified by the fact that in this scientific work we are
only interested in the first point in time when we have reached a certain target. In a practical application, of course, one needs to specify a termination criterion.

\begin{algorithm2e}%
    Generate $x \in \{0,1\}^n$ uniformly at random and $P\leftarrow\{x\}$\;
    
    \While{not stop condition}
        {
        Uniformly at random select one individual $x$ from $P$\;
        Sample $\ell$ from a binomial distribution $\mathcal{B}(n,\frac{k}{n})$\label{line:mut1}\;
        Generate $x_1,x_2,...,x_{\lambda} \in \{0,1\}^n $ each by flipping $\ell$ random bits of $x$\label{line:mut2}\;
        Select $x^+, x^- \in \{x_1,x_2,...,x_{\lambda}\}$ such that $x^+$ maximizes $f_1$ and $x^-$ maximizes $f_2$\;
        Generate $x^+_1,x^+_2,...,x^+_{\lambda} \in \{0,1\}^n $ via $\cross_c(x,x^+)$\label{line:cross1}\;
        Generate $x^-_1,x^-_2,...,x^-_{\lambda} \in \{0,1\}^n $ via $\cross_c(x,x^-)$\label{line:cross2}\;
        \For{$y \in \{x^+_1,x^+_2,...,x^+_{\lambda},x^-_1,x^-_2,...,x^-_{\lambda}\}$}
            {\If{there is no $z \in P$ such that $y\preceq z$}{$P = \{z\in P | z \npreceq y \}\cup\{y\} $;}
            }
        \Return{$P$}
        }
\caption{The \opllgsemo.}
\label{alg:opllgsemo}
\end{algorithm2e}

\subsection{Our Results} 

We analyse the \opllgsemo  algorithm by theoretical means and through experiments. Similar to the first works on the single-objective \ollga, which all regarded the \onemax benchmark, we restrict ourselves in this first analysis of the \opllgsemo to the  \oneminmax problem, which is a bi-objective version of the \onemax problem.

We show that the expected runtime (that is, the expected number of fitness evaluations until an optimal solution is
found) of our algorithm is $O\big(\big(\frac{1}{k}+\frac{1}{\lambda}\big) n^2 \log n+(k+\lambda) n^2 \big)$ when the crossover bias is taken as $c = \frac{1}{k}$, which is what the intuition given in Section \ref{sec:ollga} suggest. Consequently, quite a
broad selection of choices of $k$ and $\lambda$ leads to expected optimization times better than the classic $\Theta (n^2\log n)$. This runtime
bound suggests to take $k = \Theta(\sqrt{\log n})$ and $\lambda = \Theta(\sqrt{\log n})$, we obtain an expected optimization time of $O(n^2\sqrt{\log n})$. Note that all other choices of $\lambda \in [\omega (1), o(\log n)]$ give an asymptotically better
runtime as well, so there is some indication that this approach is useful also for problems for which analyzing the optimal
parameter choices is not possible.

The insight into the working principles of the \opllgsemo gained in the theoretical analysis can be used to design a
state-dependent choice of $\lambda$ giving an even better expected runtime. If in each iteration we chose $\lambda$ to be of order
$\sqrt{n/d}$, where $d$ is the minimum of the fitness-distance to the optimum of the two objective functions, the resulting algorithm has a quadratic expected optimization time only.

Since the state-dependent parameter choice was very successful (giving provably a quadratic expected runtime), but possibly hard to find without theoretical analyses, we also investigate a simple
self-adjusting choice of $\lambda$. To this aim, we imitate the one-fifth success rule from evolution strategies, which was independently discovered in~\cite{Rechenberg73,Devroye72,SchumerS68}. For a suitable constant $ F > 1$, we multiply $\lambda$ by $F^{1/(5n-1)}$ after each unsuccessful iteration
and we divide it by F after each iteration that found a superior solution. As we shall show, this also leads to a quadratic runtime.

\section{Runtime Analysis for Static Parameters}\label{sec:runtime1}

In this section, we conduct a rigorous runtime analysis of the \opllgsemo with static parameters on the \oneminmax benchmark.
For all runtime analysis results, we recall that the standard performance measure is the optimization time (also ``runtime'')
defined as follows.

\textbf{Definition:} The \emph{optimization time} of a MOEA $A$ on a function $f$ is the random variable $T = T (A, f )$ 
that denotes the number of fitness evaluations performed until the first time the whole Pareto front $P^\ast$ is  covered by the population of~$A$.

Observe that one iteration of the \opllgsemo requires $3\lambda$ function evaluations. Assume that we are working with a static value for $\lambda$. If $t^{\ast}$ is the first iteration after which the $\opllgsemo$ has the whole Pareto front covered by the population, then the optimization time $T$ of this run is $3t^{\ast} \lambda + 1$; recall that also the initial search point
has to be evaluated. Hence the optimization time and the first iteration $t^{\ast}$ to cover the Pareto front deviate basically by a factor of $3\lambda$. We shall thus argue with either of the two notions, but state the main results in terms of the optimization time. This will be different in Sections~\ref{sec:statedep} and~\ref{sec:runtime2}, where a varying $\lambda$ forbids this simplification.

The main result of this section is the following runtime bound, which in particular shows that our \opllgsemo  for all $k, \lambda \in [\omega(1), o(\log n)]\ $ is faster than the \gsemo on \oneminmax. 

\begin{theorem}\label{thm:main}
Let $k, \lambda \geq 2$, possibly depending on $n$. The expected optimization time of the \opllgsemo with mutation rate~$p= \frac{k}{n}$ and crossover bias $c=\frac{1}{k}$ on the \oneminmax function is
$$
O\left(\left(\frac{1}{k}+\frac{1}{\lambda}\right) n^2 \log n+(k+\lambda) n^2 \right).
$$
In particular, for both $k$ and $\lambda$ in $\Theta(\sqrt{\log n})$, the expected optimization time is of order at most $n^2 \sqrt{\log n}$.
\end{theorem}

To prove this result, we first estimate the time it takes to generate a neighbor of an existing point in the population. More precisely, in the following three lemmas, we regard the situation that $P$ contains an element $x$ with $f(x) = (n-d,d)$, but no element $y$ with $f(y) = (n-d+1,d-1)$; we shall then try to bound the time it takes until a $y$ with $f(y) = (n-d+1,d-1)$ is contained in the population. 

In the following two lemmas, analyzing separately the mutation and crossover phase, let us condition on a fixed outcome $\ell$ of the number of bits flipped in the mutation phase. We say that the mutation
phase is successful if $x$ was chosen as parent individual and at least one of the $\lambda$ offspring of $x$ has an $f_1$-value greater than $n-d-\ell$. Note that then also the mutation winner $x^+$ has such an $f_1$-value.

\begin{Lemma}\label{lem:mondom}
Then the success probability of the mutation phase is at least  $\frac{1}{n}(1-(1-\frac{d}{n})^{\lambda \ell}))$.
\end{Lemma}

\begin{proof}
We have a probability of at least $\frac{1}{n}$ of picking $x$ as a parent. Conditional on this, consider a fixed offspring. The probability that it has an $f_1$-value of $n-d-\ell$ is 
\[\binom{n-d}{\ell} / \binom{n}{\ell} \le \left(\frac{n-d}{n}\right)^\ell.\]
Since the $\lambda$ offspring are sampled independently, the probability that at least one of them has an $f_1$-value of more than $n-d-\ell$, that is, that the mutation phase is successful, is at least $1 - (1 - \frac dn)^{\ell \lambda}$.
\end{proof}

We now turn to the analysis of the crossover phase. Assuming the mutation phase to be successful, we call the crossover phase
successful if it leads to the creation of a solution $y$ with $f_1(y) = n-d+1$.

\begin{Lemma}\label{lem:crossSucc}
If the mutation phase was successful, the crossover phase is successful with probability at least $1-\left(1-c\left(1-c\right)^{\ell-1}\right)^{\lambda}$.
\end{Lemma} 

\begin{proof}
Since the mutation phase was successful, the mutation winner $x^+$ has an $f_1$-value of more than $n-d-\ell$. Consequently, there is at least one bit position out of the $\ell$ positions $x$ and $x^+$ differ in such that a crossover offspring inheriting this bit from $x^+$ and the $\ell-1$ others from $x$ has an $f_1$-value of exactly $n-d+1$. The probability that a fixed crossover offspring $x_i^+$ is of this kind is $c(1-c)^{\ell-1}$. Hence the probability that at least one crossover offspring is of this kind is
\[
\Pr[\exists i \in[\lambda]: f_1(x^+_i)) = n-d+1] \geq 1 - \left(1-c(1-c)^{\ell-1}\right)^{\lambda}.\qedhere
\]
\end{proof}

With the two lemmas above, we can now show a lower bound for the probability of finding a desired element on the Pareto front.

\begin{Lemma}\label{lem:witt}
  Assume that at a time $t$, we have $x \in P$ with $f(x) = (n-d,d)$, but there is no $y \in P$ with $f(y) = (n-d+1,d-1)$. Then there is a constant $C > 0$ such that the probability that at time $t+1$ we have such a $y$ in $P$ is at least
  \[
  p_{n-d}^+ := \frac{C}{n}\left(1-\left(\frac{d}{n}\right)^{\lambda k / 2}\right)\left(1-e^{-\lambda /(8 k)}\right).
  \]
  The expected time it takes to have such a $y$ in the population (counted from iteration $t$ on), is at most $t_{n-d}^+ := 1/p_{n-d}^+$ iterations.
\end{Lemma} 

\begin{proof}
  Let $E$ denote the event that at time $t+1$, we have a $y \in P$ with $f_1(y) = (n-d+1)$. Let $L$ be the random variable describing the value of $\ell$ sampled in line \ref{line:mut1} of Algorithm~\ref{alg:opllgsemo}. By the law of total probability, we have
  \[
  \Pr[E] \geq \sum_{\ell = \lceil k / 2\rceil}^{\lfloor 3 k / 2\rfloor} \Pr[E \mid L=\ell] \Pr[L=\ell].
  \]
  Let us denote by $E_1$ the event that the  mutation phase was successful and by $E_2$ the event that the crossover phase was successful. We note that $E \supseteq E_1 \wedge E_2$. Using Lemma \ref{lem:mondom} and \ref{lem:crossSucc}, we estimate, for a given $\ell \in [k/2..3k/2]$,
\begin{align*}
\Pr[E \mid L=\ell] 
&\ge \Pr[E_1 \mid L=\ell] \Pr[E_2 \mid E_1 \wedge  L=\ell]\\
&\geq \frac{1}{n}\left(1-\left(\frac{d}{n}\right)^{\lambda \ell}\right)\left(1-\left(1-c(1-c)^{\ell-1}\right
)^{\lambda}\right).
\end{align*}
Using the facts that $k \geq 2, c=1 / k$, and that we are only interested in values $\ell \in [k / 2 .. 3 k / 2]$, we compute
\begin{align*}
\left(1-c(1-c)^{\ell-1}\right)^{\lambda} & \leq\left(1-\frac{1}{k}\left(1-\frac{1}{k}\right)^{3 k / 2}\right)^{\lambda} \\
& \leq(1-1 /(8 k))^{\lambda} \leq e^{-\lambda /(8 k)},
\end{align*}
where we use in the second step the fact that for all $m \geq 2$ we have $(1-1 / m)^{m} \geq 1 / 4$ and in the third step that for all $m \geq 2$ we have $1 / e \geq(1-1 / m)^{m} \geq 1 /(2 e)$; in the following, we shall use these inequalities without explicit mention.
Since we are interested only in $\ell \geq k / 2$, we may estimate $\left(\frac{d}{n}\right)^{\lambda \ell} \leq\left(\frac{d}{n}\right)^{\lambda k / 2}$. Thus, in total we obtain
\[
\Pr[E \mid L = \ell] \geq\left(1-\left(\frac{d}{n}\right)^{\lambda k / 2}\right)\left(1-e^{-\frac{\lambda}{8 k}}\right)
,\]
which is independent of $\ell \in[k / 2 .. 3 k / 2]$.

Finally, it is not hard to see that $\sum_{\ell=\lceil k / 2\rceil}^{\lfloor 3 k / 2\rfloor} \operatorname{Pr}[L=\ell]$ is constant. For $k=\omega(1)$ this follows easily from Chernoff's bound. For constant $k$ we trivially have $\operatorname{Pr}[L=k]=\Theta(1)$. This proves the claim on $p_{n-d}^+$. 

The estimate for $t_{n-d}^+$ simply follows from the fact that the critical assumption that $P$ contains an $x$ with $f(x) = (n-d,d)$ is true in all future iterations as well. Hence by the above, the time to find the desired $y$ is stochastically dominated~\cite{Doerr19tcs} by a geometric random variable with success rate $p_{n-d}^+$, hence its expectation is at most $1/p_{n-d}^+$.
\end{proof}

From Lemma~\ref{lem:witt}, we now easily prove Theorem~\ref{thm:main}. 

\begin{proof}[Proof of Theorem \ref{thm:main}]
  By definition of the \opllgsemo, which is a variant of the the \gsemo, once a point of the Pareto front is covered by the population, it remains so forever. We can use this observations together with Lemma~\ref{lem:witt}, which gives an estimate on the time it takes for a neighbor of a covered point to be covered as well, to estimate the runtime by the sum of the estimates (from Lemma~\ref{lem:witt}) for the waiting times to cover an additional point. 
  
  Let us suppose that the random initial individual is $x$ with $f(x) = (n-d,d)$ for some $d > 0$. We estimate the time until for all $j \in [n-d+1..n]$ we have an individual $y$ with $f(y) = (j,n-j)$ in the population
  by
\begin{equation}\label{eq:main1}
 \sum_{j = n-d+1}^{n} t_{j-1}^+ \le \frac n C \sum_{i=1}^{n}\left(1-\left(1-\frac{i}{n}\right)^{\lambda k/2 }\right)^{-1} \left(1-e^{-\lambda /(8 k)}\right)^{-1}.
\end{equation}
For $\lambda k i > 2n$,  we have that  $\left(1-\frac{i}{n}\right)^{\lambda k / 2} \leq e^{-\frac{i\lambda k}{2n}} \leq e^{-1}$.
On the other hand, by Bernoulli’s inequality it holds that $(1 - x)^{m} \leq (1 + mx)^{-1}$ for $m \in \Z_{>0}$ and $x \in [0, 1]$. Hence for $\lambda k i \le 2 n $  we get that $\left(1-\frac{i}{n}\right)^{\lambda k /2 } \leq \left(1+\frac{\lambda k i}{2n}\right)^{-1} = 1 - \frac{\lambda k i}{2n+\lambda k i} \leq 1 - \frac{\lambda k i}{4n}$. Thus we have that 
\begin{equation}\label{eq:main2}
\begin{aligned}
\sum_{i=1}^{n}\left(1-\left(1-\frac{i}{n}\right)^{\lambda k / 2}\right)^{-1} 
&\leq \sum_{i=1}^{n}\left(1-\max\left\{ 1- \frac{\lambda k i}{4n}, e^{-1}\right\}\right)^{-1}\\
&\leq \sum_{i=1}^{n}\max\left\{\frac{4n}{\lambda k i}, (1-e^{-1})^{-1}\right\} \\
&\leq O\left(\frac{n\log(n)}{\lambda k} + n\right).
\end{aligned} 
\end{equation}
Finally, we estimate the factor $\left(1-e^{-\frac{\lambda} {8 k}}\right)^{-1}$ in~\eqref{eq:main1}. Clearly, if $\frac{\lambda} {8 k} \geq 1$, then $\left(1-e^{-\frac{\lambda} {8 k}}\right)^{-1} \ge 1 - e^{-1} = O\left(1\right)$. For $\frac{\lambda} {8 k}<1$, we use the fact that for all $x>0$ we have $\exp (-x)<1-x+\frac{x^{2}}{2}$ and bound
$$
e^{-\frac{\lambda} {8 k}}<1-\frac{\lambda}{8 k}+\frac{1}{2}\left(\frac{\lambda}{8 k}\right)^{2} \leq 1-\frac{\lambda}{8 k}+\frac{1}{2} \frac{\lambda}{8 k}=1-\frac{\lambda}{16 k}.
$$
This shows that for $\frac{\lambda}{8 k} < 1$ we have
$$
\left(1-e^{-\frac{\lambda} {8 k}}\right)^{-1}<\left(\frac{\lambda} {16 k}\right)^{-1}=O(k / \lambda).
$$
Consequently, in both cases, we have
\begin{equation}\label{eq:main3}
\left(1-e^{-\frac{\lambda} {8 k}}\right)^{-1}=O\left(\max\left\{1,\frac{\lambda}{k}\right\}\right).
\end{equation}
Replacing the terms in (\ref{eq:main1}) by (\ref{eq:main2}) and (\ref{eq:main3}) it is thus easy to see that the overall number of iterations needed is
$$
O\left(\max \left\{1, \frac{k}{\lambda}\right\}\left(\frac{n \log n}{\lambda k}+n\right)\right).
$$
Since one iteration of Algorithm \ref{alg:opllgsemo} requires $3 \lambda$ fitness evaluations, the expected time (in terms of fitness evaluations) is
\begin{align}\label{eq:main4}
O&\left(n \lambda\left(\frac{n \log n}{\lambda k}+n\right)+k n\left(\frac{n \log n}{\lambda k}+n\right)\right) \nonumber\\ &\quad=O\left(\left(\frac{1}{k}+\frac{1}{\lambda}\right) n^2 \log n+(k+\lambda) n^2\right).
\end{align}
We recall that this estimates the time to find a solution $y$ with $f(y) = (j,n-j)$ for all $j \in [n-d+1..n]$ from an initial solution $x$ with $f(x) = (n-d,d)$, and more generally, from an arbitrary initial state containing such an $x$. By symmetry, and since the above estimate does not depend on $d$,  the same estimate holds for the time to find a $y$ with $f(y) = (j,n-j)$ for all $j \in [0..n-d-1]$. Hence the entire expected runtime is the sum of these two (identical) expressions, and this finished the proof.
\end{proof}

\section{State-Dependent Parameters}\label{sec:statedep}

In this section we prove that a suitable state-dependent choice of~$\lambda$, ensuring larger $\lambda$ values towards the more difficult end of the optimization process, together with the standard parameter setting for $k$ and $c$, provably yields an asymptotic speed-up, namely an expected optimization time of $O(n^{2})$. We are not aware of any previous results showing more than a constant-factor gain through a dynamic parameter choice for a MOEA for discrete search spaces.  

To describe the quality of a state, we need the following definition.

\textbf{Definition:} Let $P$ be a population. For $b \in \{1,2\}$, let $o_b(P) := \min\{j \in[0..n-1] \mid (\exists x \in P : f_b(x) = j \wedge (\forall y \in P : f_b(y) \neq j+1)\}$. 

The idea behind this definition is to grasp the missing solution that is easiest to find. If $o_1(P) = n-d$ for some $d$, then we are in the situation of Lemma~\ref{lem:witt}, that is, we have a solution $x \in P$ with $f(x) = (n-d,d)$, but we do not have a solution $y$ with $f(y) = (n-d+1,d-1)$, and among all such pairs this pair $(x,y)$ is the one that most easily allows to generate $y$ from $x$. We then choose $\lambda$ in such a way that an expected constant number of iterations is sufficient to find $y$. 

\begin{theorem}\label{thm:fitDep}
Consider the \opllgsemo with standard parameters $p =\frac{\lambda}{n}$ and $c=\frac{1}{\lambda}$ together with a state-dependent choice of $\lambda$ such that in the beginning of each iteration $\lambda$ is set to $\lambda^* =\sqrt{\frac{n}{n-\min(o_1(P), o_2(P))}}$. Then the expected optimization time on \oneminmax is $O(n^2)$. 
\end{theorem}

\begin{proof} 
Consider one iteration of this \ollgagsemo. By symmetry, assume that $\min(o_1(P), o_2(P)) = o_1(P)$. Applying Lemma \ref{lem:witt} with $n-d = o_1(P)$ and $\lambda = k  = \lambda^*$, we see that we have a probability of $\Omega(1/n)$ to find a search point $y$ with $f(y) = (n-d+1,d-1)$ in one iteration. Consequently, we use each different value of $\lambda^*$ only for an expected number of $O(n)$ iterations. 
This bounds the entire expected optimization time by
$$
O \left(n\sum_{d=0}^{n-1} \sqrt{n /(n-d)}\right)=O\left(n\sqrt{n} \int_{1}^{n} \sqrt{1 / i} \mathrm{~d} i\right)=O(n^2) .
$$
\end{proof}

\section{Self-Adjusting Parameter Choices}
\label{sec:runtime2}

The previous section showed that a non-trivial state-dependent parameter choice gave better results than the best static parameter setting we found. The question is how an algorithm user would find this state-dependent parameter setting. Fortunately, as in the single-objective case~\cite{DoerrD18}, there is a way to let the algorithm discover a good dynamic parameter setting itself, namely via a self-adjusting choice of $\lambda$ inspired by the classic one-fifth rule (and then following the standard setting $k = \lambda/n$ and $c = 1/\lambda$). 

We first design such a self-adjusting mechanisms for the \ollgagsemo and then analyze the resulting runtime. We note that while dynamic parameter choise have been designed for MOEAs, e.g.,~
\cite{LaumannsRS01, BucheMK03, IgelHR07}, we are not aware of any such works in discrete search spaces.

To design our self-adjusting parameter setting, 
we first recall that the idea of the one-fifth rule is to adjust the parameters via multiplicative changes in such a way that roughly each fifth iteration is successful. Here success usually means that a strictly better solution was found. This definition makes not too much sense for multi-objective optimization, but for a \gsemo variant, it is natural to speak of \emph{success} if the population at the end of the iteration covers more points of the Pareto front.

We then observe that asking for a success roughly each five iterations might be too much to ask for in a \gsemo variant. Since in each iterations an offspring is generated from a random parent (chosen from a population of, here, up to $n+1$ elements), it might just take very long until a parent is chosen which has a reasonable chance to create an interesting offspring. For that reason, we rather aim at a success every roughly $5n$ iterations. With a multiplicative parameter update, this means that for a constant update strength parameter $F > 1$, we replace $\lambda$ by $\lambda /F$ in case of a success, and we replace $\lambda$ by $\lambda F^{1/ (5n-1)}$ otherwise. We set the mutation strength~$k$ and the crossover bias~$c$ depending on $\lambda$ following the standard setting, that is, $k = \lambda/n$ and $c=1/\lambda$. We allow that $\lambda$ takes non-integral values and assume that values rounded to the nearest integer are used whenever integers are required. We initialize $\lambda$ cautiously with $\lambda = 1$ and we ensure that $\lambda$ never leaves the interval $[1,\lambda]$. See Algorithm~\ref{alg:selfAdj} for the pseudo-code of this self-adjusting \opllgsemo.

\begin{algorithm2e}%
    $\lambda \leftarrow 1$\;
    Generate $x \in \{0,1\}^n$\ uniformly at random and $P\leftarrow\{x\}$\;
    \While{not stop condition}
        {
        Uniformly at random select one individual $x$ from $P$\;
        Sample $\ell$ from a binomial distribution $\mathcal{B}(n,\frac{k}{n})$\;
        Generate $x_1,x_2,...,x_{\lambda} \in \{0,1\}^n $ via randomly flipping $\ell$ bits of $x$ \label{line:mut1_2}\;
        Select $x^+, x^- \in \{x_1,x_2,...,x_{\lambda}\}$ such that $x^+$ maximizes $f_1$ and $x^-$ maximizes $f_2$\;
        Generate $x^+_1,x^+_2,...,x^+_{\lambda} \in \{0,1\}^n $ via $\cross_c(x,x^+)$ \label{line:cross1_2}\;
        Generate $x^-_1,x^-_2,...,x^-_{\lambda} \in \{0,1\}^n $ via $\cross_c(x,x^-)$\label{line:cross2_2}\;
        \For{$y \in \{x^+_1,x^+_2,...,x^+_{\lambda},x^-_1,x^-_2,...,x^-_{\lambda}\}$}
            {\If{there is no $z \in P$ such that $y\preceq z$}{$P = \{z\in P | z \npreceq y \}\cup\{y\} $}}
        \If{Success}{$\lambda\leftarrow \max\{1, \lambda / {F}\}$}
        \Else{$\lambda\leftarrow \min\{n,\lambda F^{\frac{1}{5n-1}}$\}}}
    \Return{P}
\caption{The self-adjusting \opllgsemo. We always have $k = \lambda$ and $c = 1/\lambda$. We assume that $\lambda$ is rounded to the nearest integer where an integer is required. Success means that the iteration has increased the population. }
\label{alg:selfAdj}
\end{algorithm2e}

We now conduct a runtime analysis for the self-adjusting version for our \opllgsemo and show that it can optimize the \oneminmax problem in quadratic time. 

\begin{theorem}\label{thm:square-time}
The expected optimization time of the self-adjusting \opllgsemo  on \oneminmax is $O(n^2)$ when the hyper-parameter~$F > 1$ is chosen sufficiently small.
\end{theorem}

We omit the formal proof of this result for reasons of space. A main ingredient of the proof is that the population size $\lambda$ evolved by the one-fifth success rule is usually not very far from the state-dependent
choice $\lambda^\ast$ analyzed in the previous section.

More precisely, we note that a $\lambda$ value smaller than $\lambda^*$ is not critical. Due the multiplicative update and the fact that $F$ is a constant larger than one, it takes only $O(n\lambda^*$) fitness evaluations to bring $\lambda$ up to $\lambda^*$. This ignores the possibly finding of new points of the Pareto front; since the missing point in the definition of $o_b(P)$ is the point easiest to find, this event is actually a positive one (we found a search point with a smaller value (that is, cost of one iteration) than estimated). Taking $O(n \lambda)$ fitness evaluations to adjust $\lambda$ to $\lambda^*$ is uncritical, since even with the optimal value $\lambda = \lambda^*$ we allow for $O(n)$ iterations, hence $O(n\lambda^*)$ fitness evaluations, to find the desired individual. 

Once we have $\lambda \ge \lambda^*$, as computed in the proof of Theorem~\ref{thm:fitDep}, we have an $\Omega(1/n)$ probability to find the desired individual. Hence the probability to fail for $\gamma n$ iterations, is $\exp(-\Omega(\gamma))$. Such a sequence of failures increases the $\lambda$ value by a factor of $(F^{1/(5n-1)})^{\gamma n}$. If $F$ is chosen sufficiently small (but larger than one), then the cost incurred by this too high value of $\lambda$ is outnumbered by the small probability of $\exp(-\Omega(\gamma))$ of this negative event.

\section{Experiments}\label{sec:experiments}

\begin{figure}
\begin{center}
\includegraphics[width=0.8\textwidth]{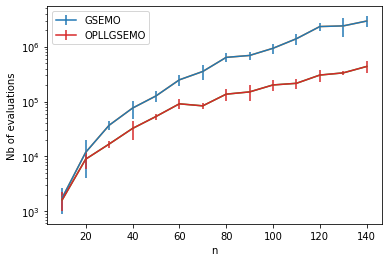}
\caption{The mean number of function evaluations with standard deviation (in $10$ independent runs) of the \gsemo and the \opllgsemo on \oneminmax.}\label{fig:results}
\end{center}
\end{figure}

To see if the asymptotic runtime differences of the algorithms regarded in this work are visible already for realistic problem sizes, we implemented the algorithms and ran them on the \oneminmax problem of size $n = 10,20,...,140$.

For the \opllgsemo we use the standard parameter setting $k = \lambda = 1/c$ with $\lambda = 7\log n$. Unfortunately, we could not find parameters that led to an interesting performance of the self-adjusting \opllgsemo.

The average runtimes of the \gsemo and \opllgsemo are displayed in Figure~\ref{fig:results}. The superiority of the \opllgsemo is clearly visible, being more than a factor of $5$ faster for the largest problem sizes.

\section{Conclusion}\label{sec:conclusion}

In this work, we showed how to incorporate the main building block of the single-objective \ollga algorithm into a MOEA, namely the GSEMO algorithm. Our mathematical runtime analysis on the \oneminmax benchmark showed that it profits from the same speed-ups that the \ollga did on the \onemax benchmark, and this for fixed parameters, parameters depending on the spread of the population, and a self-adjusting parameter choice (where the latter two needed some modification compared to the \ollga). 

Given these positive results, a natural continuation of this line of research is to see which other advantages of the \ollga transfer to the multi-objective case. Given the good performance of the \ollga on multimodal problems, an interesting next test problem could be the multimodal multi-objective benchmark designed in~\cite{DoerrZ21aaai}. Given that the \ollga showed a good performance with heavy-tailed parameter choices~\cite{AntipovBD20gecco, AntipovBD21gecco}, both on simple and multimodal problems, seeing if these advantages continue into the multi-objective world is an equally interesting direction for future research.

From a broader perspective, this work also shows that the central idea of the \ollga can be encapsulated into new complex mutation operator, which can easily be combined with existing algorithms. We are optimistic that this idea, here only done for the \gsemo, can be profitable also for other algorithms, let it be multi-objective ones such as the NSGA-II or classic EAs.

Overall, this work shows that it can be interesting to try to transfer the recent theory developments in the better understood single-objective world into the theory of MOEAs.

\subsection*{Acknowledgement}

This work was supported by a public grant as part of the
Investissements d'avenir project, reference ANR-11-LABX-0056-LMH,
LabEx LMH.

\newcommand{\etalchar}[1]{$^{#1}$}


}
\end{document}